\RequirePackage{fix-cm}
\documentclass[smallcondensed, natbib]{svjour3}
\smartqed
\usepackage{epsfig}
\usepackage{algorithm}
\usepackage{algpseudocode}
\usepackage{amsfonts}
\usepackage{amsmath}
\usepackage{amssymb}
\usepackage{caption}
\usepackage{mathptmx}
\newcommand*{\affaddr}[1]{#1}
\newcommand*{\affmark}[1][*]{\textsuperscript{#1}}
\journalname{Machine Learning}
\begin{document}

\title{Stochastic Divergence Minimization for Biterm Topic Model}
\author{Zhenghang Cui\affmark[1] \and
        Issei Sato\affmark[1, 2] \and \\
	Masashi Sugiyama\affmark[2, 1]}
\authorrunning{Zhenghang Cui et al.}
\institute{
  Zhenghang Cui \\
  cui@ms.k.u-tokyo.ac.jp \\ \\
  Issei Sato \\
  sato@k.u-tokyo.ac.jp \\ \\
  Masashi Sugiyama \\
  sugi@k.u-tokyo.ac.jp \\ \\
  \affaddr{\affmark[1]The University of Tokyo, Japan}\\
  \affaddr{\affmark[2]RIKEN, Japan}\\
}
\date{Received: date / Accepted: date}
\maketitle

\begin{abstract}
As the emergence and the thriving development of social networks,
a huge number of short texts are accumulated and need to be processed.
Inferring latent topics of collected short texts is useful for understanding its hidden structure and predicting new contents.
Unlike conventional topic models such as latent Dirichlet allocation (LDA),
a biterm topic model (BTM) was recently proposed for short texts to overcome the sparseness of document-level word co-occurrences
by directly modeling the generation process of word pairs.
Stochastic inference algorithms based on collapsed Gibbs sampling (CGS) and collapsed variational inference have been proposed for BTM.
However,
they either require large computational complexity, or rely on very crude estimation.
In this work,
we develop a stochastic divergence minimization inference algorithm for BTM to estimate latent topics more accurately in a scalable way.
Experiments demonstrate the superiority of our proposed algorithm compared with existing inference algorithms.

\end{abstract}

\section{Introduction}
As social network services are dominant in people's daily life,
a huge number of short text data has been accumulated.
At the same time,
other data which can be found on traditional web pages,
such as article titles or public forum comments can also be regarded as possessing the same attribute of short length.
It would be an essential and interesting task to explore their inner structure for a wide range of applications,
such as classification based on contents,
or prediction for future documents that have not emerged yet.
Because of the document level word co-occurrence sparsity caused by short document length,
conventional topic models such as probabilistic latent semantic indexing (pLSA) \citep{plsa} or latent Dirichlet allocation (LDA) \citep{lda}
fail to show favorable inference performance on data sets consisting of short texts.
A biterm topic model (BTM) \citep{btm} was proposed to alleviate this problem caused by document level word co-occurrence sparsity.
Instead of each single word,
the generation process of each unordered combination of two words, or a \textit{biterm},
is modeled in BTM.
Each biterm is assumed to be assigned with one topic.
Compared to conventional topic models,
this modification makes BTM less sensitive to the shortness of each document,
and more stable to clearly reveal the relationship between words.
By modeling the word co-occurrences explicitly and combining words into biterms,
it has been shown by experiments \citep{btm} that BTM successfully alleviates the problem caused by document level word co-occurrence sparsity and
keeps the generality and flexibility at the same time.

For inferring model parameters and estimating latent topics for BTM,
a batch inference algorithm based on collapsed Gibbs sampling (CGS) is first proposed together with the model \citep{btm} to approximate the true posterior distribution of parameters.
Based on this batch CGS inference algorithm,
two online algorithms are proposed \citep{btm} to scale up for data sets of large size.
One online algorithm is based on the idea of updating hyperparameters between time slices,
which is inspired by the online LDA algorithm \citep{online-lda},
while the other online algorithm is based on the idea of resampling topics of observed biterms for sufficient times after a new biterm is observed,
which is inspired by an incremental Gibbs sampler for LDA \citep{inc-lda}.
On the other hand,
based on the idea of zero-order stochastic collapsed variational Bayesian inference (SCVB0) for LDA \citep{foulds},
a similar SCVB0 algorithm for BTM was proposed for better latent topics estimation \citep{naru}.
However, these online algorithms are either not working very efficiently on memory usage,
or relying on very crude estimation.

In this paper,
we propose a stochastic divergence minimization (SDM) inference algorithm for BTM based on minimizing the $\alpha$-divergence to estimate latent topics more accurately.
First,
inspired by the work for LDA \citep{cvb},
we reconstruct collapsed variational Bayesian inference which uses only the zero-order Taylor series approximation (CVB0) as an optimization problem of $\alpha$-divergence minimization.
Then,
we apply a stochastic approximation method to this optimization problem to develop a stochastic inference algorithm.

For a general probabilistic model,
CGS inference algorithms try to find a posterior distribution,
while variational Bayesian (VB) inference algorithms try to find a closest distribution within a function family. \citep{beal}
The closeness is usually measured by the KL-divergence.
VB transforms the original inference problem to an optimization problem,
which can be solved by a simple gradient descent algorithm.
Similarly to the manipulation in CGS,
collapsed variational Bayesian (CVB) marginalizes out unconcerned parameters and only infer latent parameters.
For example,
CVB for BTM \citep{naru} marginalized out model parameters,
which form a vector indicating the topic proportion and the matrix indicating the word distribution for each topic,
and only calculated the posterior distribution for latent parameters indicating topic assignments of biterms.
CGS algorithms usually converge slower and are strongly influenced by the initial state of parameters due to the inner characteristic of a Monte Carlo Markov Chain sampling algorithm.
On the other hand,
CVB is a deterministic algorithm.
Empirically,
it converges faster and performs better \citep{Asuncion}.

Since exact evaluation of expectations in the CVB formula is intractable,
the idea of using only the zero-order term of its Taylor series as a rough approximation is appealing.
This results in a zero-order CVB inference algorithm (CVB0) proposed for BTM \citep{naru}.
Based on CVB0,
stochastic approximations are developed to scale up the algorithm for huge data sets \citep{naru}.
However,
the reason why zero-order approximation is used instead of higher order approximations is not clearly explained.
Furthermore,
although the SCVB0 for the BTM algorithm utilizes a scale coefficient to reduce the computational complexity of each iteration from $\mathcal{O}(W)$ to $\mathcal{O}(1)$,
where $W$ denotes the size of the vocabulary,
the risk of arithmetic underflow in floating point calculations always exists when processing data sets of very large size.
It also utilizes a very crude approximation for essential statistics at each iteration.

\paragraph{Contributions}
Considering the issues discussed above,
we propose a novel SDM inference algorithm for BTM.
We have three main contributions listed as follows.

\begin{itemize}
 \item We provide a novel formulation of SCVB0 inference for BTM from the perspective of $\alpha$-divergence minimization.
       This provides a new means to understand the inner attribute of SCVB0 inference for BTM.
       This is inspired by the similar work developed for LDA \citep{cvb}.
 \item We derive an SDM algorithm for BTM based on the $\alpha$-divergence minimization formulation of SCVB0.
       SDM for BTM is an one-pass algorithm,
       which means it processes each biterm only once and stops when all biterms have been processed.
       Compared to SCVB0 for BTM,
       SDM for BTM requires the same amount of memory and has the same computational complexity for processing a single biterm.
       On the other hand,
       SCVB0 does not preserve the sufficient statistics for the counts of each word
       and
       has the risk of arithmetic underflow in floating point calculations.
       SDM for BTM does not have these problems
       and
       provides better approximation.
       Experiments reveal that SDM for BTM can estimate latent topics more accurately,
       and thus can predict documents which have not emerged yet with higher accuracy than existing methods.
 \item We analyze the convergence of our proposed method by using Martingale convergence theory.
\end{itemize}

The remainder of this paper is organized as follows.
In Section 2, we introduce the related works on BTM, its existing inference algorithms,
and the theoretical background for SDM.
In Section 3, we introduce our proposed SDM algorithm.
In Section 4, we conduct experiments to evaluate our proposed method against existing methods and discuss the result.
In Section 5, we conclude this paper.

\section{Related Works}
In this section, we will introduce the biterm topic model (BTM) \citep{btm},
followed by its batch and online inference algorithms.
Essential information for $\alpha$-divergence is presented at the end of this section.

\subsection{BTM}
Conventional topic models such as LDA usually fail to show satisfactory performance on short text data sets.
BTM was proposed to alleviate this problem by modifying the word generating part of the graphical model.
Instead of modeling the generation of each word,
BTM directly models the generation of biterms,
which are unordered combinations of two words.
For example,
a document of $n$ words will generate $\binom n2$ combinations of two words.
Compared to conventional topic models,
this modification makes BTM less sensitive to the short length of each document,
and biterms are more stable to clearly reveal the relationship between words.
Based on the original paper \citep{btm},
the notation is listed as follows.

\begin{itemize}
 \item A data set contains $N_B$ biterms,
       where each biterm is denoted by $b_i = \{w_{i1}, w_{i2}\}$.
 \item The number of topics is denoted by $K$.
 \item The size of vocabulary is denoted by $W$.
 \item A topic proportion vector is denoted by $\theta$.
       Its length is $K$ and all of its entries sum to $1$.
 \item A word distribution matrix is denoted by $\Phi$.
       Its size is $K \times W$.
       Each row vector $\phi_k$ has length $W$ and sums to $1$.
 \item A topic indicator variable for biterm $b_i$ is denoted by $z_i$.
       It has a length of $K$ and all of its entries sum to $1$.
\end{itemize}

The generative process is described formally as follows.

\begin{enumerate}
 \item Draw $\theta$ $\sim$ Dirichlet($\gamma$)
 \item For each topic $k$
 \begin{enumerate}
  \item Draw $\phi_k$ $\sim$ Dirichlet($\beta$)
 \end{enumerate}
 \item For each biterm $b_i$
 \begin{enumerate}
  \item Draw $z_i$ $\sim$ Multinomial($\theta$)
  \item Draw $w_{i1}, w_{i2}$ $\sim$ Multinomial($\phi_{z_i}$)
 \end{enumerate}
\end{enumerate}

Here,
Dirichlet($\gamma$) denotes a Dirichlet distribution with parameter $\gamma$,
and Multinomial($\theta$) denotes a multinomial distribution with parameter $\theta$.
The graphical model of BTM is shown in Fig. \ref{fig:btm}.
\begin{figure}
 \centering
 \epsfig{figure=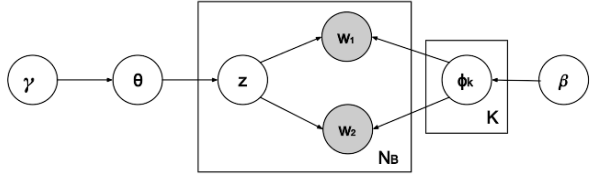,width=.7\textwidth}
 \caption{The graphical model of BTM.
 $\gamma$ and $\beta$ are hyperparameters for the prior Dirichlet distributions of model parameters $\theta$ and $\Phi$.
 $z$ denotes hidden parameters for each biterm indicating assigned topics.
 $w_1$ and $w_2$ are observed words of a biterm.}
 \captionsetup{justification=centering}
 \label{fig:btm}
\end{figure}

Following the generation process,
we can express the likelihood of a data set $B$ conditioned on model parameters $\theta$ and $\Phi$ as

\begin{equation}
 P(B|\theta, \Phi) = \prod_{i=1}^{N_B} \sum_{k=1}^K \theta_k \phi_{k, w_{i1}} \phi_{k, w_{i2}} .
\end{equation}

\subsection{Batch Inference Algorithm}
Here,
we will concisely introduce the batch inference algorithm which estimates all of the three parameters,
$z$ which indicates the topic assignments,
$\theta$ which indicates the topic proportion and $\Phi$ which indicates the word distribution for each topic.
Since it is intractable to compute the exact posterior distributions of these parameters,
the CGS algorithm is used to approximate the true posterior distributions \citep{btm}.
Parameters $\theta$ and $\Phi$ are first integrated out using conjugate priors,
then $z_i$ for each biterm $b_i$ is sampled using the posterior distribution conditioned on all of the other variables.
After processing all biterms,
$\theta$ and $\Phi$ can be restored using $z$.
However, this can be a computational burden when the size of given data set is large,
which motivates the development of stochastic inference algorithms that will be discussed in Section \ref{sec:obtm} and Section \ref{sec:ibtm}.
The following formula is used to sample $z_i$ for each biterm $b_i$:

\begin{equation}
 \label{gibbs}
 P(z_i=k|z_{\backslash i}, B) \propto
 (n_{\backslash i, k} + \gamma)
 \frac{(n_{\backslash i, w_{i1}|k}+\beta)(n_{\backslash i, w_{i2}|k}+\beta)}
 {(n_{\backslash i,\cdot|k}+W\beta)(n_{\backslash i,\cdot|k}+W\beta + 1)} .
\end{equation}

Let $z_{\backslash i}$ be the whole topic assignment vector without considering $b_i$,
$n_{\backslash i, k}$ be the count of biterms assigned to topic $k$ without counting $b_i$ and
$n_{\backslash i, w|k}$ be the count of times that word $w$ is assigned to topic $k$ without counting $b_i$.
The dot in $n_{\backslash i, \cdot|k}$ means taking the sum over all words.
After a sufficient number of iterations over the whole data set,
we can restore $\theta$ and $\Phi$ using following formulas:

\begin{equation}
 \label{btm-phi}
 \phi_{k,w} = \frac{n_{w|k} + \beta}{n_{\cdot|k}+W\beta},
\end{equation}
\begin{equation}
 \label{btm-theta}
 \theta_k = \frac{n_k + \gamma}{N_B + K\gamma},
\end{equation}
where $n_k$ is the number of biterms assigned to topic $k$ and $n_{w|k}$ is the count of times that word $w$ is assigned to topic $k$.
The dot in $n_{\cdot|k}$ means taking the sum over all words.

\subsection{Online BTM Algorithm}
\label{sec:obtm}
In recent real-world inference problems,
the size of data to analysis is usually very large and keeps increasing.
To deal with such large data,
it would be useful to develop algorithms that can handle data in the streaming form.
In the original paper \citep{btm},
two kinds of algorithms have been introduced to deal with data sets of very large size.
The online BTM algorithm will be introduced here and the incremental BTM algorithm will be introduced in Section \ref{sec:ibtm}.

The idea of the online BTM algorithm is inspired by the similar algorithm proposed for LDA \citep{online-lda}.
The data set is supposed to be separated in multiple time-slices,
e.g., hourly, daily or weekly.
Within the processing of a single time-slice sample,
hyperparameters $\gamma$ and $\beta$ are updated using statistics of data in this time slice.
After a sufficient number of iterations,
parameters $\theta$ and $\Phi$ can be restored to reflect the influence of this time slice.

The notations are described as follows.
A biterm set of time $t$ is denoted by $B^{(t)}$.
The number of biterms assigned to topic $k$ within $B^{(t)}$ is denoted by $n_k^{(t)}$.
The number of times word $w$ is assigned to topic $k$ within $B^{(t)}$ is denoted by $n_{w|k}^{(t)}$.
Hyperparameters for $\theta$ are denoted by vector $\{\gamma_1, \ldots, \gamma_k \}$ and
hyperparameters for $\Phi$ are denoted by matrix $\{\beta_1, \ldots, \beta_K \}$,
where $\beta_k$ is a vector consisting of $\{ \beta_{k, 1}, \ldots, \beta_{k, W}\}$.
The conditional distribution for sampling each topic $z_i$ is given by

\begin{equation}
 \label{online-gibbs}
 P(z_i=k|z_{\backslash i}^{(t)}, B^{(t)}, \gamma^{(t)}, \beta^{(t)}) \propto
  (n_{\backslash i, k}^{(t)} + \gamma_k^{(t)})
  \frac{(n_{\backslash i, w_1|k}^{(t)} + \beta_{k, w_1}^{(t)}) (n_{\backslash i, w_2|k}^{(t)} + \beta_{k, w_2}^{(t)})}
  {[\sum_{w=1}^W(n_{\backslash i, w|k}^{(t)} + \beta_{k, w}^{(t)})] [\sum_{w=1}^W(n_{\backslash i, w|k}^{(t)} + \beta_{k, w}^{(t)}) + 1]}.
\end{equation}

After the processing of each time-slice sample,
hyperparamters can be updated as

\begin{equation}
 \label{obtm-gamma}
 \gamma_k^{(t+1)} = \gamma_k^{(t)} + \lambda n_k^{(t)},
\end{equation}

\begin{equation}
 \label{obtm-beta}
 \beta_{k,w}^{(t+1)} = \beta_{k,w}^{(t)} + \lambda n_{w|k}^{(t)},
\end{equation}
where the decay weight is denoted by $\lambda \in [0,1]$.
It controls the dependency to data in past time slices.
The details of the procedure are described in Alg. \ref{alg:obtm}.

\begin{algorithm}
\caption{Online BTM Algorithm}
\label{alg:obtm}
\begin{algorithmic}
 \State Set $\gamma^{(1)} = \{\gamma, \ldots, \gamma\}$, $\{\beta_k = \{\beta, \ldots, \beta\}\}_{k=1}^K$.
 \For{$t$ = 1 to $T$}
  \For{$b_i \in B^{(t)}$}
   \State Sample a topic for $b_i$ using a uniform distribution
  \EndFor
  \For{iterations}
   \For{$b_i \in B^{(t)}$}
    \State Sample a topic for $b_i$ using the distribution defined by Eq. \eqref{online-gibbs}
    \State Update statistics of $n_k^{(t)}$, $n_{w_1|k}^{(t)}$ and $n_{w_2|k}^{(t)}$
    \State Update $\gamma$ and $\{\beta_k\}_{k=1}^K$ using Eq. \eqref{obtm-gamma} and Eq. \eqref{obtm-beta}.
   \EndFor
  \EndFor
  \State Compute $\theta^{t}$ and $\Phi^{(t)}$ using Eq. \eqref{btm-theta} and Eq. \eqref{btm-phi}.
 \EndFor
\end{algorithmic}
\end{algorithm}

\subsection{Incremental BTM Algorithm}
\label{sec:ibtm}
Although the online BTM algorithm can be adapted to sequential data,
updating parameters immediately after a biterm arrived may be essential in some situations.
The incremental BTM algorithm is proposed for this purpose.
It can update parameters after the arrival of each single biterm.

The idea of the incremental BTM algorithm is inspired by the incremental Gibbs sampler \citep{inc-lda}.
Specifically,
the main task is that after the arrival of a new biterm,
when the routine of sampling its topic ends,
a biterm sequence called a rejuvenation sequence will be constructed on the run and the topic of all biterms belonging to this sequence will be resampled.
Apparently,
the length and the choice of the rejuvenation sequence would influence the performance profoundly.
For convenience,
the sequence length is regarded as a hyperparameter and the uniform distribution is used to generate it.

The details of the procedure are described in Alg. \ref{alg:ibtm}.

\begin{algorithm}
\caption{Incremental BTM Algorithm}
\label{alg:ibtm}
\begin{algorithmic}
 \For{$b_i \in B$}
  \State Sample a topic for $b_i$ using Eq. \eqref{gibbs}.
  \State Update statistics of $n_k$ and $n_{w|k}$
  \State Generate rejuvenation sequence $R$
  \For{$b_j \in R$}
   \State Sample a topic for $b_j$ using Eq. \eqref{gibbs}.
   \State Update statistics of $n_k$ and $n_{w|k}$
  \EndFor
 \EndFor
 \State Compute global parameters $\theta$ and $\Phi$ using Eq. \eqref{btm-theta} and Eq. \eqref{btm-phi}
\end{algorithmic}
\end{algorithm}

\subsection{SCVB0 Algorithm for BTM}
The batch algorithm, CVB0 for BTM, will be introduced following its stochastic formulation,
SCVB0 for BTM.

CVB0 for BTM \citep{naru} is inspired by CVB0 for LDA \citep{Asuncion}.
Similarly to CGS,
global parameters $\theta$ and $\Phi$ are first marginalized out and only inference for latent parameter $z$ is performed.
A zero-order approximation of Taylor series is utilized because some expectations are intractable to evaluate.
The updating formula for variational parameter $z_{i,k}$ can be deducted as

\begin{equation}
 \label{scvb0-sampling}
 z_{i,k} \propto (N_{\backslash i, k} + \alpha)
  \frac{(N_{\backslash i, w_{i1}|k} + \beta) (N_{\backslash i, w_{i2}|k} + \beta)}
  {(2\, N_{\backslash i, k} + W\beta) (2\, N_{\backslash i, k} + W\beta + 1)},
\end{equation}
where $N_k = \sum_{b_i\in B} z_{i, k}$,
$N_{w|k} = \sum_{b_i\in B_w} z_{i, k}$,
$B_w$ denotes the set of biterms containing word $w$ and $\backslash i$ means counting without considering $b_i$.

SCVB0 for BTM is based on the idea of ignoring the subtraction of the current biterm
and update statistics in a stochastic way.
Storing all variational parameters is not necessary and a very crude estimation of $N_k$ and $N_{w|k}$ when a biterm $b_i$ is observed can be expressed by

\begin{equation}
 \hat N_k = |B|\, z_{i,k},
\end{equation}

\begin{equation}
 \hat N_{w|k} =
  \begin{cases}
    |B|\,z_{i,k} & \mathrm{if} \,\,w\in b_i, \\
    0 & \mathrm{otherwise} .
  \end{cases}
\end{equation}

Then, $N_k$ and $N_{w|k}$ can be updated using the following formulas:

\begin{equation}
 N_k \leftarrow (1-\rho_t) N_k + \rho_t \hat N_k,
\end{equation}

\begin{equation}
 N_{w|k} \leftarrow (1-\rho_t) N_{w|k} + \rho_t \hat N_{w|k},
\end{equation}
where $\rho_t = 1/(t+\tau)^\kappa$ denotes the step size.

To reduce the computational complexity of each update from $\mathcal{O}(W)$ to $\mathcal{O}(1)$,
the following technique is used to represent the value of $N_{w|k}$.
A scaling coefficient $a$ and a dummy matrix $A_{w|k}$ are in fact stored,
where $N_{w|k} = a\,A_{w|k}$ is satisfied.
Every time when $N_{w|k}$ is updated,
one just needs to multiply $a$ by $(1-\rho_t)$ and manually computes the values of $A_{w_{i1}|k}$ and $A_{w_{i2}|k}$.
This manipulation significantly reduces the computational complexity,
but bears the risk of $a$'s underflow,
because $(1-\rho_t)$ is multiplied repeatedly during the algorithm.

After processing all of the biterms,
global parameters can be restored using the following formulas:

\begin{equation}
 \label{scvb0-theta}
 \theta_k \propto N_k + \alpha,
\end{equation}

\begin{equation}
 \label{scvb0-phi}
 \phi_{k, w} \propto N_{w|k} + \beta.
\end{equation}

The details of the procedure are described in Alg. \ref{alg:scvb}.

\begin{algorithm}
\caption{SCVB0 BTM Algorithm}
\label{alg:scvb}
\begin{algorithmic}
 \For{$b_i \in B$}
  \For{each topic $k$}
   \State Compute $z_{i,k}$ using Eq. \eqref{scvb0-sampling}
   \State Update $N_{k}$ and $N_{w|k}$
  \EndFor
 \EndFor
 \State Compute global parameters $\theta$ and $\Phi$ using Eq. \eqref{scvb0-theta} and Eq. \eqref{scvb0-phi}
\end{algorithmic}
\end{algorithm}

\subsection{$\alpha$-divergence}
Here we briefly introduce the concepts of $\alpha$-divergence and local divergence projection inference.
More details can be found in \citet{amari} and \citet{minka}.

\paragraph{Definition}
The $\alpha$-divergence can be perceived as a generalized KL divergence.
We will denote its detailed definition using two distributions $p(x)$ and $q(x)$.
The $\alpha$-divergence from $p(x)$ to $q(x)$, indexed by $\alpha \in (-\infty, \infty)$,
is defined as

\begin{equation}
 D_{\alpha}[p||q] = \frac{\int\alpha p(x)+(1-\alpha)q(x)-p(x)^{\alpha}q(x)^{1-\alpha}dx}{\alpha(1-\alpha)}.
\end{equation}

Notice that $p(x)$ and $q(x)$ need not to be normalized before calculating the $\alpha$-divergence.
Some useful special cases of $\alpha$ are:

\begin{equation}
 D_{-1}[p||q] = \frac12\int\frac{(q(x)-p(x))^2}{p(x)}dx,
\end{equation}

\begin{equation}
 \lim_{\alpha \to 0}D_{\alpha}[p||q] = \mathrm{KL}[q||p],
\end{equation}

\begin{equation}
 \lim_{\alpha \to 1}D_{\alpha}[p||q] = \mathrm{KL}[p||q].
\end{equation}

\paragraph{Local $\alpha$-divergence projection}
Suppose that the distribution $q(x)$ we approximate can be fully factorized.
That is, $q(x) = \prod_{i=1}^n q(x_i)$, where $x_i$ denotes the $i$-th element of the vector $x$.
$x = (x_1, x_2, \ldots, x_n)^\top$, where $\top$ denotes the transpose.
Depending on $p(x)$,
it is intractable to naively compute the $\alpha$-divergence.
To avoid this problem,
we focus on each single $x_i$ and then optimize each $q(x_i)$ by

\begin{equation}
 \textrm{argmin}_{q(x_i)} D_{\alpha}[p(x_i|x_{\backslash i})q(x_{\backslash i})||q(x)],
\end{equation}
where $x_{\backslash i}$ represents all but $i$-th entry of $x$ and $q(x) = q(x|x_{\backslash i})q(x_{\backslash i})$.
Its update formula can be obtained by taking the derivative of the $\alpha$-divergence and equating it to zero:

\begin{equation}
 q(x_i) \propto \mathbb{E}_{q(x_{\backslash i})}
  \left[ \left(\frac{p(x)}{q(x_{\backslash i})}\right)^{\alpha} \right]
   ^{\frac1{\alpha}}.
\end{equation}
Since the expectation over $q(x_{\backslash i})$ is computationally intractable,
we approximate it by first splitting the numerator as

\begin{equation}
 q(x_i) \propto \mathbb{E}_{q(x_{\backslash i})}
  \left[\left(p(x_i|x_{\backslash i})\frac{p(x_{\backslash i})}{q(x_{\backslash i})}\right)^{\alpha}\right]^{\frac1{\alpha}}.
\end{equation}
We then substitute $p(x_{\backslash i})$ with $q(x_{\backslash i})$ to obtain the following approximation:

\begin{equation}
 q(x_i) \propto \mathbb{E}_{q(x_{\backslash i})}[(p(x_i|x_{\backslash i}))^{\alpha}]^{\frac1{\alpha}}.
\end{equation}

This is the method we will use in the derivation of a stochastic divergence minimization algorithm to approximate the $\alpha$-divergence.

\section{Proposed Method}
In this section,
we propose a novel SDM inference algorithm for BTM.
We will first show the derivation of SDM.
Then we will show its relation to the leave-one-out likelihood (LOO).

\subsection{Derivation of Divergence Minimization}
We assume the independence between latent topics of the given biterms as

\begin{equation}
 q(B, z) = \prod_{i=1}^{N_B}q(b_i, z_i).
\end{equation}
We then estimate this distribution by $\alpha$-divergence minimization:

\begin{equation}
 q^*(B, z) = \textrm{argmin}_{q(B, z)} \textrm{D}_{\alpha}[p(B, z)||q(B, z)],
\end{equation}
where $q^*(B, z)$ denotes the optimized distribution.

Since it is intractable to compute this minimization,
we consider the following local divergence minimization.
First, noting that

\begin{equation}
 p(b_i,z_i=k) \propto
  (n_{\backslash i, k}+\gamma)
  \frac{(n_{\backslash i, w_{i1}|k}+\beta)(n_{\backslash i, w_{i2}|k}+\beta)}
  {(n_{\backslash i, \cdot | k}+W\beta) (n_{\backslash i, \cdot | k}+W\beta + 1)},
\end{equation}
where the notations are the same as those in Eq. \eqref{gibbs} for CGS.

We then reparameterize $q(b_i, z_i)$ as follows:

\begin{equation}
 \label{sdm-q}
 q(b_i, z_i=k)\propto \frac{a_k^{\backslash i}b_{k, w_{i1}}^{\backslash i}b_{k, w_{i2}}^{\backslash i}}{c_k^{\backslash i}(c_k^{\backslash i} + 1)},
\end{equation}

\begin{equation}
 a_k^{\backslash i} = \tilde n_{\backslash i, k}+\gamma,
\end{equation}
\begin{equation}
 b_{k, w}^{\backslash i} = \tilde n_{\backslash i, w|k}+\beta,
\end{equation}
\begin{equation}
 c_k^{\backslash i} = \tilde n_{\backslash i, \cdot | k}+W\beta.
\end{equation}
Notice that $\tilde n_{\backslash i, k}$, $\tilde n_{\backslash i, w|k}$ and $\tilde n_{\backslash i, \cdot | k}$ here are not counts.
They are just parameters of the function defined above.

We also define

\begin{equation}
 q^{\backslash a}(b_i,z_i=k) = \frac{b_{k, w_{i1}}^{\backslash i}b_{k, w_{i2}}^{\backslash i}}{c_k^{\backslash i}(c_k^{\backslash i} + 1)},
\end{equation}
\begin{equation}
 q^{\backslash b_1}(b_i,z_i=k) = \frac{a_k^{\backslash i}b_{k, w_{i2}}^{\backslash i}}{c_k^{\backslash i}(c_k^{\backslash i} + 1)},
\end{equation}
\begin{equation}
 q^{\backslash b_2}(b_i,z_i=k) = \frac{a_k^{\backslash i}b_{k, w_{i1}}^{\backslash i}}{c_k^{\backslash i}(c_k^{\backslash i} + 1)},
\end{equation}
\begin{equation}
 q^{\backslash c}(b_i,z_i=k) = a_k^{\backslash i}b_{k, w_{i1}}^{\backslash i}b_{k, w_{i2}}^{\backslash i}.
\end{equation}

Recalling that the $\alpha$-divergence does not need the distributions to be normalized,
we can define the following local projections:

\begin{equation}
 \label{alpha-a}
 (a_k^{\backslash i})^*=
 \textrm{argmin}_{a_k}D_{\alpha}
 [(n_{\backslash i, k}+\gamma)q^{\backslash a, i}(B,z)||
 a_k^{\backslash i}q^{\backslash a, i}(B,z)],
\end{equation}
\begin{equation}
 \label{sdm-b1}
 (b_{k, w_{i1}}^{\backslash i})^*=\textrm{argmin}_{b_{k,w}}D_{\alpha}[(n_{\backslash i, w_{i1}|k}+\beta)q^{\backslash b_1, i}(B,z)||b_{k,w_{i1}}^{\backslash i}q^{\backslash b_1, i}(B,z)],
\end{equation}
\begin{equation}
 \label{sdm-b2}
 (b_{k, w_{i2}}^{\backslash i})^*=\textrm{argmin}_{b_{k,w}}D_{\alpha}[(n_{\backslash i, w_{i2}|k}+\beta)q^{\backslash b_2, i}(B,z)||b_{k,w_{i2}}^{\backslash i}q^{\backslash b_2, i}(B,z)],
\end{equation}
\begin{equation}
 (c_k^{\backslash i})^*=
  \textrm{argmin}_{c_k}D_{\alpha}
  [\frac{q^{\backslash c, i}(B,z)}{(n_{\backslash i, \cdot | k}+W\beta)(n_{\backslash i, \cdot | k}+W\beta + 1)}||\frac{q^{\backslash c, i}(B,z)}{c_k^{\backslash i}(c_{i}^{\backslash i} + 1)}],
\end{equation}
where

\begin{equation}
 q^{\backslash a, i}(B, z)=q^{\backslash a}(b_i,z_i) q(B_{\backslash i}, z_{\backslash i}),
\end{equation}
\begin{equation}
 q^{\backslash b_1, i}(B, z)=q^{\backslash b_1}(b_i,z_i)q(B_{\backslash i}, z_{\backslash i}),
\end{equation}
\begin{equation}
 q^{\backslash b_2, i}(B, z)=q^{\backslash b_2}(b_i,z_i)q(B_{\backslash i}, z_{\backslash i}),
\end{equation}
\begin{equation}
 q^{\backslash c, i}(B, z)=q^{\backslash c}(b_i,z_i)q(B_{\backslash i}, z_{\backslash i}).
\end{equation}

Taking the derivative of Eq. \eqref{alpha-a} with respect to $a^{\backslash i}$ and equating it to zero yields

\begin{equation}
 \sum_{z_{\backslash i}} q^{\backslash a, i}(b_i,z) - \frac{\sum_{z_{\backslash i}}(n_{\backslash i, k}+\gamma)^{\alpha}q^{\backslash a,i}(b_i,z)}{(a_k^{\backslash i})^{\alpha}} = 0.
\end{equation}
With $\sum_{z_{\backslash i}}q(z_{\backslash i})=1$,
we can obtain

\begin{equation}
 \begin{split}
  a_k^{\backslash i} & = [\sum_{z_{\backslash i}}(n_{\backslash i, k}+\gamma)^{\alpha}q(z_{\backslash i})]^{\frac1{\alpha}} \\
  & = \mathbb{E}_{q(z_{\backslash i})}[(n_{\backslash i, k}+\gamma)^{\alpha}]^{\frac1{\alpha}}.
 \end{split}
\end{equation}

Similarly, we can derive the solutions to the other optimization problems listed above as follows:

\begin{equation}
 b_{k, w}^{\backslash i} = \mathbb{E}_{q(z_{\backslash i})}[(n_{\backslash i, w|k}+\beta)^{\alpha}]^{\frac1{\alpha}},
\end{equation}
\begin{equation}
 c_k^{\backslash i} = \mathbb{E}_{q(z_{\backslash i})}
  \left[\left(
	 \frac1{n_{\backslash i, \cdot | k}+W\beta}
	\right)^{\alpha}\right]^{\frac1{\alpha}}.
\end{equation}

If we use $\alpha$-divergence projection with $\alpha =1$ for $a_k^{\backslash i}$ and $b_{k, w}^{\backslash i}$,
while using it with $\alpha=-1$ for $c_k^{\backslash i}$, we can obtain the update formula for $q(z_i)$ as

\begin{equation}
 \label{sdm-q-result}
 q(b_i,z_i=k)\propto
 (\mathbb{E}[n_{\backslash i, k}]+\gamma)
 \frac{(\mathbb{E}[n_{\backslash i, w_{i1}|k}]+\beta)(\mathbb{E}[n_{\backslash i, w_{i2}|k}]+\beta)}
 {(\mathbb{E}[n_{\backslash i, \cdot | k}]+W\beta)(\mathbb{E}[n_{\backslash i, \cdot | k}]+W\beta + 1)},
\end{equation}
which can also be obtained by SCVB0.

\subsection{Relation to LOO Likelihood}
Here,
we investigate the relationship between divergence minimization and the LOO likelihood.

For the LOO prediction of a new biterm $b_i$, its probability is given by

\begin{equation}
 \begin{split}
 p(b_i|B_{\backslash i}) & = \sum_{k=1}^K \mathbb{E}_{p(z_{\backslash i}|B_{\backslash i})} [p(b_i, z_i=k)|B_{\backslash i}, z_{\backslash i}] \\
  & = \sum_{k=1}^K \mathbb{E}_{p(z_{\backslash i}|B_{\backslash i})}
  \left[\frac{n_{\backslash i, k}+\gamma}{(N_B-1)+K\gamma}
   \frac{(n_{\backslash i, w_{i1}|k}+\beta)(n_{\backslash i, w_{i2}|k}+\beta)}{(n_{\backslash i, \cdot|k}+W\beta(n_{\backslash i, \cdot|k}+W\beta + 1))}\right].
 \end{split}
\end{equation}
On the other hand,
using Eq. \eqref{sdm-q-result}, we get

\begin{equation}
 \begin{split}
  q(b_i) & = \sum_{k=1}^K q(b_i, z_i=k) \\
  & = \sum_{k=1}^K
  \frac{\mathbb{E}[n_{\backslash i, k}+\gamma]}{(N_B-1)+K\gamma}
  \frac{(\mathbb{E}[n_{\backslash i, w_{i1}|k}]+\beta)(\mathbb{E}[n_{\backslash i, w_{i2}|k}]+\beta)}
  {(\mathbb{E}[n_{\backslash i, \cdot | k}]+W\beta)(\mathbb{E}[n_{\backslash i, \cdot | k}]+W\beta + 1)}.
 \end{split}
\end{equation}
This full likelihood is similar to the above LOO likelihood and can be regarded as an approximation to it.
Considering the close relation between the LOO likelihood and the full likelihood,
lowering the full likelihood indicated by $q(b_i)$ can result in a lower likelihood of the data set.
The close correlation between the LOO perplexity and test perplexity of LDA has been shown with detailed experimental results in \citet{sdm}.

\subsection{Derivation of SDM}
For the three terms we defined in $\alpha$-divergence minimization,
we have $\mathbb{E}[n_{\backslash i|k}] = 2\,\mathbb{E}[n_{\backslash i,\cdot|k}]$.
Therefore, $a_k^{\backslash i}$ can be restored from $c_k^{\backslash i}$ and we need only to compute the values of $b_k^{\backslash i}$.
$c_k^{\backslash i}$ can be calculated as $c_k^{\backslash i} = \sum_{w=1}^Wb_{k, w}^{\backslash i}$,
so we can update it by $c_k^{\backslash i} \leftarrow c_k^{\backslash i} - (b_{k,w}^{\backslash i})^{(\textrm{old})} + (b_{k,w}^{\backslash i})^{(\textrm{new})}$.
For this reason,
we will focus on the stochastic approximation for the term $b_{k,w}^{\backslash i}$.

Recall that $b_{k,w}^{\backslash i} = \mathbb{E}[n_{\backslash i, w|k}]+\beta$.
We can rewrite it in the form of fixed point iteration with step size $\rho_t$:

\begin{equation}
 \begin{split}
  (b_{k,w}^{\backslash i})^{(t+1)} & = (1-\rho_t) (b_{k,w}^{\backslash i})^{(t)} + \rho_t(\mathbb{E}[n_{\backslash i,w|k}]+\beta) \\
  & = (b_{k,w}^{\backslash i})^{(t)} + \rho_t(\mathbb{E}[n_{\backslash i,w|k}] + \beta - (b_{k,w}^{\backslash i})^{(t)}).
 \end{split}
\end{equation}

We can then replace the term $\mathbb{E}[n_{\backslash i,w|k}]$ with approximation $\tilde n_{\backslash i,w|k}$,
which is defined by

\begin{equation}
 \tilde n_{\backslash i,w|k} = (n_w - 1)q(z_{i'} = k|b_{i'}),
\end{equation}
where $i' \neq i$ is a random sample with $b_{i'}$ also containing the word $w$.
We can then substitute it into the update formula:

\begin{equation}
 \label{sdm-b}
 (b_{k,w}^{\backslash i})^{(t+1)} = (b_{k,w}^{\backslash i})^{(t)} + \rho_t(\tilde n_{\backslash i,w|k} + \beta - (b_{k,w}^{\backslash i})^{(t)}).
\end{equation}

This update formula is actually a realization of update based on word vocabulary.
Furthermore, we process all the biterms in an one-pass fashion,
which means that each biterm is processed only once one by one until the end.
For this reason,
we focus on the perspective of word type update and reformulate the update as

\begin{equation}
 \label{sdm-update-b}
 b_{k,w}^{t(w)+1} = b_{k,w}^{t(w)} + \rho_{t(w)}(\tilde n_{\backslash i,w|k} + \beta - (b_{k,w}^{\backslash i})^{t(w)}),
\end{equation}
where $t(w)$ represents how many times the word $w_j$ has been updated so far and $\rho_{t(w)} = 1/(1+t(w))^\kappa$.
In the context of stochastic approximation,
we can simply make $b_{k,w}^{t(w)+1} = (b_{k,w}^{\backslash i})^{(t+1)}$.

The detailed procedure is described in Alg. \ref{alg:sdm}.

\begin{algorithm}
 \caption{SDM BTM Algorithm}
 \label{alg:sdm}
 \begin{algorithmic}
  \State Randomly initialize $b_{k,w}$ and compute $c_{k}$
  \For{$b_i \in B$}
   \For{$w \in b_i$}
    \State Update $b_{k,w}^{t(w)}$ using Eq. \eqref{sdm-update-b}
    \State Update $c_w$
    \State Increment $t(w)$
   \EndFor
  \EndFor
  \State Compute global parameters $\theta$ and $\Phi$.
 \end{algorithmic}
\end{algorithm}

\section{Experiments}
We evaluate the effect of our proposed algorithm compared with existing inference algorithms.

\subsection{Experimental Settings}
We used the data set called \textit{Tweets2011} (http://trec.nist.gov/data/tweets) to evaluate the algorithms.
Tweets2011 is a standard collection of tweets published between January 23rd and February 8th, 2011.

The raw data of Tweets2011 is very noisy and contains tweets in multiple languages.
Many of the languages are difficult to perform morphological analysis such as Japanese.
Therefore, in the preprocessing stage,
we applied a filter to keep only English tweets.
Then, we implemented essential preprocessing tasks such as stop word removal and punctuation removal.
Finally, we removed documents consisting of only a single word since they can not form biterms and have no co-occurrence between different words.

After preprocessing,
as we can see from the document length distribution shown in Fig. \ref{fig:length},
most of the documents have length less than 10.
There are around $3.5$ million documents and the average document length is $6.94$ words.

\begin{figure}
 \centering
 \epsfig{figure=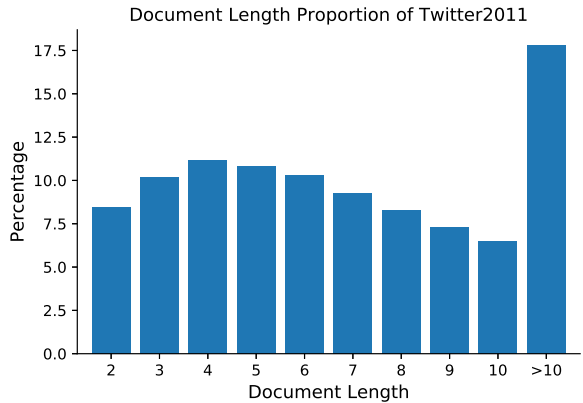,width=.7\textwidth}
 \caption{Document length distribution in Tweets2011.
 Most of the documents are shorter than 10 words.}
 \captionsetup{justification=centering}
 \label{fig:length}
\end{figure}

In most real world situations,
the size of a given data set is too large to run a batch algorithm.
Therefore, to better simulate practical circumstances,
we considered the four algorithms, namely,
the proposed SDM algorithm,
the SCVB0 algorithm for BTM,
the online BTM algorithm,
and the incremental BTM algorithm,
that can scale up to large data sets.

In order to show the convergence speed of each algorithm,
and to compare all of the above algorithms at the same scale,
we let the programs output the result together with the proportion of processed biterms from the beginning of each experiment.
All of the experiments are conducted on an Ubuntu server with 2.9 GHz Intel Xeon E5-2667 CPU and 64 GB memory.

After the preprocessing,
we obtained $3.5$ million short documents.
We then formulated around $95$ million biterms from these short documents.
We shuffled all of the biterms and divided them into two data sets:
the training set and the test set.
Their sizes are set to be around $4:1$.
It turns out that the training set contains about $75$ million biterms,
and the test set contains about $20$ million biterms.

For three existing methods,
hyperparameters are set in the same way as the original paper \citep{btm}.
To assure both the CGS algorithms finish at a realistic time,
we set the length of the rejuvenation sequence to $10$ for the incremental BTM algorithm and
the inner iteration times to $10$ for the online BTM algorithm.
For all of the algorithms,
we set $\gamma$ to be $50/K$ and $\beta$ to be $0.01$.
As step size hyperparameters,
$\tau = 1000$ and $\kappa = 0.8$ are used for SCVB0,
while $\kappa = 0.51$ is used for SDM.

\citet{btm} chose to evaluate the coherence of observed topics to measure the performance of each algorithm,
which requires us to use external data from sources such as Wikipedia for evaluating the pointwise mutual information.
This result can be dependent on the size and the quality of the external data.
Therefore,
we decided to use the strategy called predictive sample re-use \citep{psr} to evaluate the efficiency of each algorithm.
This evaluation strategy simply measures the log likelihood sum of the test data set based on the model parameters calculated by each algorithm.

\subsection{Evaluation}
We demonstrate and discuss the performance of four inference algorithms.
Each set of experiments is repeated 10 times.

\begin{figure}
 \centering
 \epsfig{figure=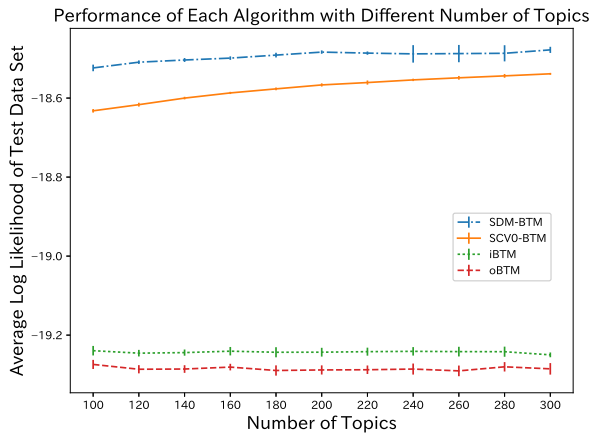,width=.7\textwidth}
 \caption{Average test-set likelihood after processing all data once.}
 \captionsetup{justification=centering}
 \label{fig:all}
\end{figure}

\begin{figure}
 \centering
 \epsfig{figure=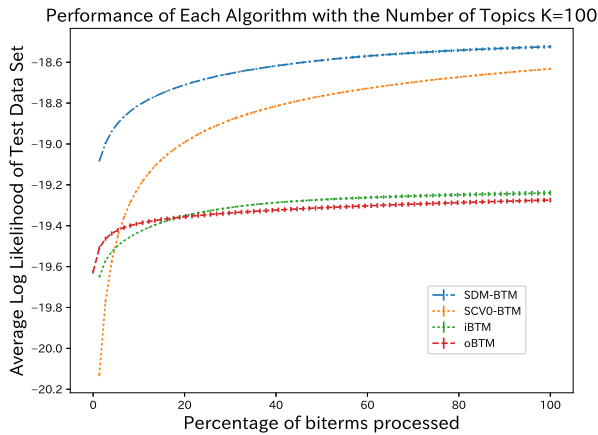,width=.7\textwidth}
 \caption{Average test set likelihood trend with the number of topics K = 100.}
 \captionsetup{justification=centering}
 \label{fig:case}
\end{figure}

In Fig. \ref{fig:all},
the horizontal axis represents experimental settings under different topic numbers.
In Fig. \ref{fig:case},
the horizontal axis represents the percentage of documents that have been processed.
In both of the figures,
the vertical axis represents the mean value of the average log likelihood of the test set,
which is the higher the more favorable.
Their standard deviations are also shown by error bars over 10 runs.
SDM-BTM represents the proposed SDM algorithm for BTM.
SCVB0-BTM represents the SCVB0 algorithm for BTM.
iBTM represents the incremental BTM algorithm.
oBTM represents the online BTM algorithm.

The comparison of the update cost, memory and purpose between different algorithms is shown in Tab. \ref{table:res},
where $R$ denotes the length of the rejuvenation sequence and $B_t$ denotes the size of a biterm mini-batch.

\begin{table}[h]
\caption{Comparison Between Algorithms.}
\label{table:res}
\centering
 \begin{tabular}{|c|c|c|c|c|}
  \hline
  & Update Cost & Memory & Purpose \\ [0.5ex]
  \hline
  iBTM & $\mathcal{O}(R)$ & $\mathcal{O}(K(1+W)+B_t)$ & Post. Approx.\\
  \hline
  oBTM & $\mathcal{O}(B_t(1+KW))$ & $\mathcal{O}(K(1+W)+N_B)$ & Post. Approx.\\
  \hline
  SCVB0-BTM & $\mathcal{O}(K)$ & $\mathcal{O}(K(1+W))$ & Post. Approx. \\
  \hline
  SDM-BTM & $\mathcal{O}(K)$ & $\mathcal{O}(K(1+W))$ & LOO Est.\\
  \hline
 \end{tabular}
\end{table}

As shown in Fig. \ref{fig:all} and Fig. \ref{fig:case},
compared to online algorithms based on CGS,
SDM-BTM shows a higher convergence speed and a better convergence result.
This is because SDM-BTM is a deterministic algorithm while CGS algorithms are based on sampling.
Compared to SCVB0-BTM,
SDM-BTM performs better on a real world data set because SDM-BTM
preserves the sufficient statistics correctly
and
is not exposed to the risk of arithmetic underflow.

\section{Conclusions}
In this paper,
we first reviewed the BTM and its existing inference algorithms.
We then reconstructed CVB0 inference for BTM and proposed a novel SDM inference algorithm,
which is a stochastic inference algorithm that can be applied to practical circumstances.
It outperformed existing methods in our experiments.

For future work,
it would be interesting and essential to explore the relationship between the number of topics and the performance.
Developing new inference algorithms based on other $\alpha$-divergences or
conducting experiments on various data sets and studying their results would also be important challenges.


\bibliographystyle{spbasic}
\bibliography{ms}

\section*{Appendix: Convergence Proof}
In this appendix,
we will show that the proposed SDM algorithm does converge to the minimization of the divergence.
This is not obvious since the term used in Eq. \eqref{sdm-b} is not a gradient of the divergence in Eq. \eqref{sdm-b1} or Eq. \eqref{sdm-b2}.

We will proof the convergence of the stochastic approximation by Eq. \eqref{sdm-b} for the optimization problem in Eq. \eqref{sdm-b1}.
The proof for Eq. \eqref{sdm-b2} is similar, thus it would be omitted.
The whole process is similar to but different from the appendix of \citet{sdm}.

First,
we repeat the definition of the stochastic optimization problem based on Martingale convergence theory.
The optimization problem is to find $b^* = \textrm{argmin}_b f(b)$ and
the update formula is $b^{(t+1)} = b^{(t)} + \rho_t s^{(t)}$.
We also let $\mathcal{F}^{(t)}$ to be the history of the variable sequence,
which is also called \textit{filtration} :

\begin{equation}
 \mathcal{F}^{(t)} = \{b^{(1)}, \cdots, b^{(t)}, s^{(1)}, \cdots, s^{(t-1)}, \rho_1, \cdots, \rho_t\}.
\end{equation}

Then we have the Martingale theory as follows.

\begin{theorem}
 \label{theorem}
 Assume step size $\rho_t$, function $f$ and stochastic search direction $s^{(t)}$ satisfy following four conditions.
 \begin{enumerate}
  \item Step size $\rho_t$ is a non-negative scalar and satisfies
	\begin{equation}
	 \lim_{t\to \infty}\rho_t =0,\, \sum_{t=1}^{\infty}\rho_t=\infty,\,\sum_{t=1}^{\infty}(\rho_t)^2<\infty.
	\end{equation}
  \item Function $f$ is continuously differentiable and there exists come constant $L$ such that
	\begin{equation}
	 ||\nabla f(b) - \nabla f(\bar b)|| \leq L||b-\bar b||, \forall b,\, \bar b \in \mathbb{R}^K.
	\end{equation}
  \item There exists a positive constant $C$ such that
	\begin{equation}
	 C\nabla f(b^{(t)})^\top\mathbb{E}[s^{(t)}|\mathcal{F}^{(t)}]\leq-||\nabla f(b^{(t)})||^2,\,\forall t>0.
	\end{equation}
  \item There exist positive constants $A$ and $B$ such that
	\begin{equation}
	 \mathbb{E}[||s^{(t)}||^2|\mathcal{F}^{(t)}]\leq A+B||\nabla f(b^{(t)})||^2,\, \forall t>0.
	\end{equation}
 \end{enumerate}
 Then the update equation $b^{(t+1)} = b^{(t)} + \rho_t s^{(t)}$ satisfies the three holds with probability one.
 \begin{enumerate}
  \item The sequence $f(b^{(t)})$ converges.
  \item $\lim_{t\to\infty}\nabla f(b^{(t)})=0$.
  \item Every limit point of $b^{(t)}$ is a stationary point of $f$.
 \end{enumerate}
\end{theorem}

The detailed proof of above theory can be found in the super-Martingale convergence theorem \citep{theory}.

First, we reform the stochastic update formula.
Given the objective function
\begin{equation}
 \begin{split}
  D_1(b_{k, w_{i1}}^{\backslash i})
  & = D_1[(n_{\backslash i, w_{i1}|k}+\beta)q^{\backslash b_1, i}(B,z)||b_{z_i,w_{i1}}^{\backslash i}q^{\backslash b_1, i}(B,z)] \\
  & = D_1[(n_{\backslash i, w_{i1}|k}+\beta)q^{\backslash b_1}(b_i,z_i)q(B_{\backslash i}, z_{\backslash i})||
  b_{z_i,w_{i1}}^{\backslash i}q^{\backslash b_1}(b_i,z_i)q(B_{\backslash i}, z_{\backslash i})] \\
 \end{split}
\end{equation}

We take its derivatives regarding to $b_{k, w_{i1}}^{\backslash i}$:

\begin{equation}
 \label{partial}
 \frac{\partial}{\partial b_{k, w_{i1}}^{\backslash i}} D_1(b_{k, w_{i1}}^{\backslash i}) =
 - \frac{q^{\backslash b_1}(b_i, z_i)}{b_{k, w_{i1}}^{\backslash i}}
 (\mathbb{E}[n_{\backslash i, w_{i1}|k}] + \beta - b_{k, w_{i1}}^{\backslash i}).
\end{equation}

Next, we define the stochastic direction

\begin{equation}
 \begin{split}
  s_{i', k} & = n_{\backslash i, w_{i1}} q(z_{i'}=k|b_{i'}) + \beta - b_{k, w_{i1}}^{\backslash i}\\
  & = \mathbb{E}[n_{\backslash i, w_{i1}|k}] + \beta - b_{k, w_{i1}}^{\backslash i} + \xi_{i',k},
 \end{split}
\end{equation}
where $\xi_{i', k} = n_{\backslash i, w_{i1}} q(z_{i'}=k|b_{i'}) - \mathbb{E}[n_{\backslash i, w_{i1}|k}]$ and
$b_{i'}$ is another biterm that contains the word $w_{i1}$.
Then we rewrite the stochastic direction as

\begin{equation}
 s_{i', k} = \frac{b_{k, w_{i1}}^{\backslash i} c_k^{\backslash i}(c_k^{\backslash i}+1)}{a_k^{\backslash i}b_{k, w_{i2}}^{\backslash i}}
  \frac{\partial}{\partial b_{k, w_{i1}}^{\backslash i}} D_1(b_{k, w_{i1}}^{\backslash i}) + \xi_{i', k}.
\end{equation}

The step size we actually use satisfies the first condition of the theorem \ref{theorem} and
it is not hard to show that the objective function $D_1(b_{k, w_{i1}}^{\backslash i})$ of minimization satisfies the second condition.
Therefore,
what left is to show the satisfaction of the third and forth condition.
From the paper that introduced stochastic divergence minimization for LDA \citep{sdm},
the proof of following lemma can be found.

\begin{lemma}
 \label{xi}
 If the stochastic noise term $\xi^{(t)}$ satisfies the following conditions,
 then the stochastic direction $s^{(t)}$ satisfies the third and forth condition of the theorem \ref{theorem}.
 \begin{enumerate}
  \item $\{\xi^{(t)}\}$ is a Maringale difference sequence with respect to filtration $\mathcal{F}^{(t)}$,
	which means that $\mathbb{E}[\xi^{(t)}|\mathcal{F}^{(t)}]=0,\,\forall t>0$.
  \item $\xi^{(t)}$ has bounded variance.
	For example, it is square integrable with
	\begin{equation}
	 \mathbb{E}[||\xi^{(t)}||^2|\mathcal{F}^{(t)}]\leq C(1+||\nabla D_1(b_{k, w_{i1}}^{(t)})||^2),\,\forall t>0,
	\end{equation}
	for some constant $C$.
 \end{enumerate}
\end{lemma}

Therefore, what left is to show the following lemma.

\begin{lemma}
 The noise term $\xi^{(t)}$ satisfies the two conditions listed in the Lemma \ref{xi}.
\end{lemma}

\begin{proof}
 For the first condition, we show that
 \begin{equation}
  \begin{split}
   \mathbb{E}[\xi_k^{(t)}|\mathcal{F}^{(t)}] & = \mathbb{E}_{i'}[\xi_{i', k}] \\
   & = n_{\backslash i, w_{i1}}\mathbb{E}_{i'}[q(z_{i'}=k|b_{i'})\delta(w_{i1}\in b_{i'})] - \mathbb{E}[n_{\backslash i, w_{i1}|k}] \\
   & = \sum_{i'\neq i} q(z_{i'}=k|b_{i'})\delta(w_{i1}\in b_{i'}) - \mathbb{E}[n_{\backslash i, w_{i1}|k}] \\
   & = 0.
  \end{split}
 \end{equation}

Therefore the first condition is satisfied.
Recall Eq. \eqref{partial}:
$$
 \frac{\partial}{\partial b_{k, w_{i1}}^{\backslash i}} D_1(b_{k, w_{i1}}^{\backslash i}) =
 - \frac{q^{\backslash b_1}(b_i, z_i)}{b_{k, w_{i1}}^{\backslash i}}
 (\mathbb{E}[n_{\backslash i, w_{i1}|k}] + \beta - b_{k, w_{i1}}^{\backslash i}).
$$

We then define stochastic gradient as
\begin{equation}
 \frac{\partial}{\partial b_{k, w_{i1}}^{\backslash i}} \tilde D_1(b_{k, w_{i1}}^{\backslash i}) =
 - \frac{q^{\backslash b_1}(b_i, z_i)}{b_{k, w_{i1}}^{\backslash i}}
 (n_{\backslash i, w_{i1}}q(z_{i'}=k) + \beta - b_{k, w_{i1}}^{\backslash i}).
\end{equation}

We confirm that
\begin{equation}
 \mathbb{E}_{i'}[\frac{\partial}{\partial b_{k, w_{i1}}^{\backslash i}} \tilde D_1(b_{k, w_{i1}}^{\backslash i})] =
  \frac{\partial}{\partial b_{k, w_{i1}}^{\backslash i}} D_1(b_{k, w_{i1}}^{\backslash i}),
\end{equation}

The difference between the stochastic gradient and the real gradient is
\begin{equation}
 \begin{split}
  & \frac{\partial}{\partial b_{k, w_{i1}}^{\backslash i}} \tilde D_1(b_{k, w_{i1}}^{\backslash i}) - \frac{\partial}{\partial b_{k, w_{i1}}^{\backslash i}} D_1(b_{k, w_{i1}}^{\backslash i}) \\
  & = \frac{q^{\backslash b_1}(b_i, z_i)}{b_{k, w_{i1}}^{\backslash i}}(n_{\backslash i, w_{i1}}q(z_{i'}=k) - \mathbb{E}[n_{\backslash i, w_{i1}|k}]) \\
  & = \frac{q^{\backslash b_1}(b_i, z_i)}{b_{k, w_{i1}}^{\backslash i}}\xi_{i', k} \\
  & = \frac{a_k^{\backslash i}b_{k, w_{i2}}^{\backslash i}}{b_{k, w_{i1}}^{\backslash i}c_k^{\backslash i}(c_k^{\backslash i} + 1)}\xi_{i', k},
 \end{split}
\end{equation}
thus there exists a constant $C$ so that

\begin{equation}
 \xi_{i', k} \leq C
  \{\frac{\partial}{\partial b_{k, w_{i1}}^{\backslash i}} \tilde D_1(b_{k, w_{i1}}^{\backslash i}) - \frac{\partial}{\partial b_{k, w_{i1}}^{\backslash i}} D_1(b_{k, w_{i1}}^{\backslash i}) \}.
\end{equation}

Thus
\begin{equation}
 \begin{split}
  & \mathbb{E}[||\xi^{(t)}||^2|\mathcal{F}^{(t)}] \\
  & \leq C^2\mathbb{E}[||\nabla \tilde D_1(b_{k, w_{i1}}^{\backslash i}) - \nabla D_1(b_{k, w_{i1}}^{\backslash i})||^2|\mathcal{F}^{(t)}] \\
  & \leq C^2\mathbb{E}[||\nabla \tilde D_1(b_{k, w_{i1}}^{\backslash i}) ||^2|\mathcal{F}^{(t)}].
 \end{split}
\end{equation}

We then introduce $D_1^*(b_{k, w_{i1}}^{\backslash i}) \geq \tilde D_1(b_{k, w_{i1}}^{\backslash i})$ which is given by

\begin{equation}
 \frac{\partial}{\partial b_{k, w_{i1}}^{\backslash i}} D_1^*(b_{k, w_{i1}}^{\backslash i}) =
  \frac{q^{\backslash b_1}(b_i, z_i)}{b_{k, w_{i1}}^{\backslash i}} (n_{\backslash i, w_{i1}} + \beta - b_{k, w_{i1}}^{\backslash i}).
\end{equation}

Therefore, there exists another constant $G$ such that
\begin{equation}
 n_{\backslash i, w_{i1}} + \beta - b_{k, w_{i1}}^{\backslash i} \leq G(\mathbb{E}[n_{\backslash, w_{i1}|k}] + \beta - b_{k, w_{i1}}^{\backslash i}),
\end{equation}

for example,
\begin{equation}
 ||\nabla D_1^*(b_{k, w_{i1}}^{\backslash i,(t)})||^2 \leq G^2||\nabla D_1(b_{k, w_{i1}}^{\backslash i, (t)})||^2.
\end{equation}

Therefore, we can say
\begin{equation}
 \begin{split}
  \mathbb{E}[\xi^{(t)}||\mathcal{F}^{(t)}] & \leq C^2\mathbb{E}[||\nabla \tilde D_1(b_{k, w_{i1}}^{\backslash i,(t)})||^2|\mathcal{F}^{(t)}] \\
  & \leq C^2G^2||\nabla D_1(b_{k, w_{i1}|k}^{\backslash i, (t)})||^2,
 \end{split}
\end{equation}

which satisfies the second condition.
\qed
\end{proof}

To conclude, the convergence of our stochastic approximation is proved.

\end{document}